\documentclass{article}
\usepackage{authblk}
\usepackage[american]{babel}
\usepackage{natbib}
\usepackage{mathtools}
\usepackage{booktabs}
\usepackage{tikz}
\usepackage{graphicx}
\usepackage{multirow}
\usepackage{float}
\usepackage{amsthm}
\usepackage{subcaption}
\usepackage{hyperref}

\newtheorem{proposition}{Proposition}

\theoremstyle{definition}

\title{\textbf{One Simple Trick to Fix Your Bayesian Neural Network}}

\author[1,2,3]{Piotr Tempczyk}
\author[2,4]{Ksawery Smoczyński}
\author[1]{Philip Smolenski-Jensen}
\author[1,5]{Marek Cygan}

\affil[1]{%
Institute of Informatics, University of Warsaw
}
\affil[2]{%
Polish National Institute for Machine Learning (\url{opium.sh})
}
\affil[3]{%
\url{deeptale.ai}
}
\affil[4]{Faculty of Mathematics and Computer Science, Adam Mickiewicz University}
\affil[5]{%
Nomagic
}
  
\begin{document}
\maketitle

\begin{abstract}
One of the most popular estimation methods in Bayesian neural networks (BNN)\citep{blundell2015weight} is mean-field variational inference (MFVI)\citep{blei2017variational}.
In this work, we show that neural networks with ReLU\citep{fukushima1975cognitron} activation function induce posteriors, that are hard to fit with MFVI.
We provide a theoretical justification for this phenomenon, study it empirically, and report the results of a series of experiments to investigate the effect of activation function on the calibration of BNNs.
We find that using Leaky ReLU activations\citep{maas2013rectifier} leads to more Gaussian-like weight posteriors and achieves a lower expected calibration error (ECE)\citep{guo2017calibration} than its ReLU-based counterpart.
\end{abstract}

\section{Introduction}

Uncertainty estimation and neural network calibration are important aspects of real-world deep learning applications\citep{wang2019aleatoric,lakshminarayanan2017simple,loquercio2020general}. Especially, when dealing with out-of-distribution or noisy examples, which is often the case in robotics or autonomous driving. Currently ensemble methods provide best approaches to uncertainty estimation and are often superior to Bayesian methods, especially to variational inference, which is the fastest one (e.g. compared to MCMC), but also the most inaccurate one. This paper makes approach to change current state of knowledge and make MFVI methods more scalable and accurate in terms of uncertainty estimation.

We start with observing that for neural networks with  ReLU activation every parameter $w_i$ has an unbounded range of values for which the likelihood function $\mathcal{L}_n$ (for a single data point $x_n$, $n=1\dots N$) is constant and greater than 0. %

\begin{proposition}
\label{prop:likelihood}
For every neuron with a ReLU activation function and non-negative input with weight $w_i$ there exists upper bound $w_{ni}^*$, such that if $w_i \leq w_{ni}^{*}$ then $\frac{\partial\mathcal{L}_n}{\partial w_i} = 0$.
\end{proposition}
\begin{proof}
We can decompose this derivative using chain rule into:
$\frac{\partial\mathcal{L}_n}{\partial w_i} = \frac{\partial\mathcal{L}_n}{\partial z}\frac{\partial z}{\partial a}\frac{\partial a}{\partial w_i},$
where $a = w_i h_i + \sum_{j \neq i} w_j h_j + b$ is a selected neuron output, $z=\max(0, a)$ is a ReLU activation function and $h_k$ is $k$-th input to the neuron.
$\frac{\partial\mathcal{L}}{\partial w_i}$ vanishes if any of the terms in the chain is equal to $0$. 
Note, that if $h_i = 0$, then $\frac{\partial z}{\partial a}=0$, for any value of $w_i$.
Otherwise, if $h_i>0$ then let $w_{ni}^*\coloneqq -\frac{\sum_{j \neq i} w_j h_j + b}{h_i}$, which implies that $a \leq 0$, which in turns gives $\frac{\partial z}{\partial a}=0$.
\end{proof}

For neural networks with ReLU activations $h_i \geq 0$ holds for all the layers except the first one. For the whole dataset $\mathcal{L} = \prod_{n=1}^N\mathcal{L}_n$. This fact and Proposition~\ref{prop:likelihood} implies that there is $w_i^* = \displaystyle\min_{n}{w_{ni}^*}$, such that if $w_i \leq w_{i}^{*}$ then $\frac{\partial\mathcal{L}}{\partial w_i} = 0$. Therefore, for all the weights and any dataset, there exists an infinite plateau on the loss function in some part of the weight space. The posterior resulting from $\mathcal{L}$ might be impossible to normalize (especially for improper uniform priors) and may cause problems when fitting it with MFVI. This can be to some extent alleviated by using proper priors, but the problem does not vanish because the resulting posterior's shape may still be far from Gaussian's and hence lead to poor approximation when using MFVI. 

In this work we investigate this problem and propose a simple solution, where we change all activation functions from ReLU to LeakyReLU and optimize its \emph{negative slope} parameter. We show this leads to posteriors much more suitable for fitting with Gaussian distribution without deteriorating the accuracy.

\begin{figure}[t!]
\centering
\includegraphics[width=0.9\textwidth]{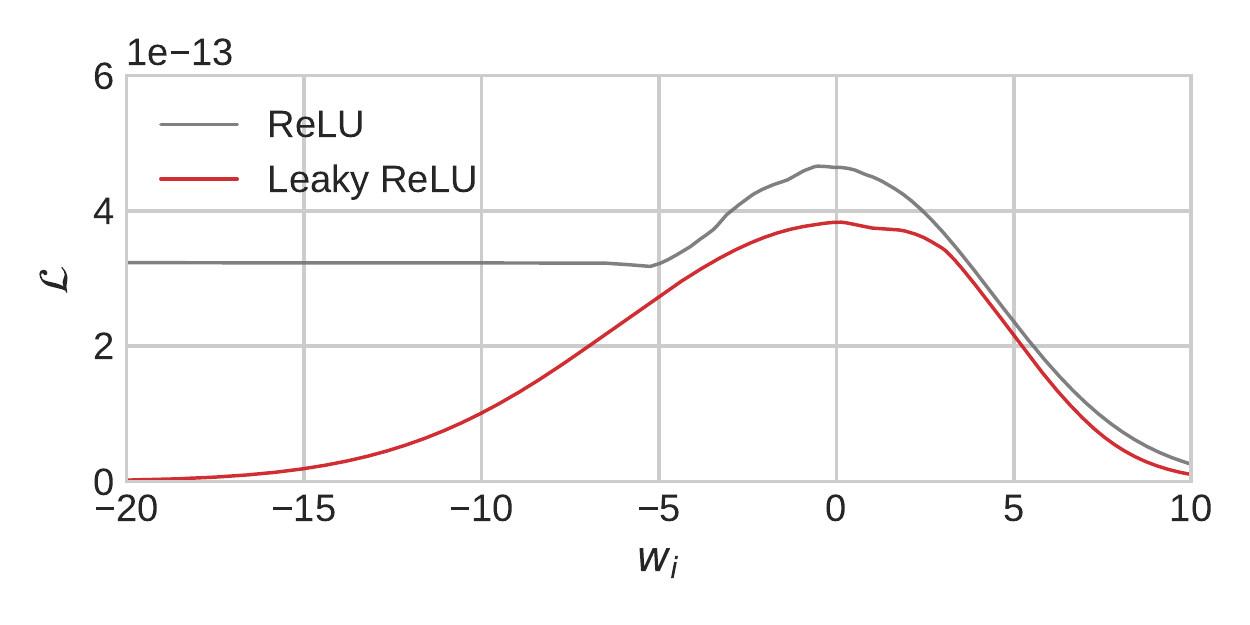}
\caption{The conditional likelihood for the weight $w_i$ from a fully connected layer from a model trained on 600 samples from MNIST with ReLU and Leaky ReLU activations.}
\label{fig:ll}
\end{figure}

\section{Experiments}

\begin{figure}[b!]
    \centering
    \includegraphics[width=0.9\textwidth]{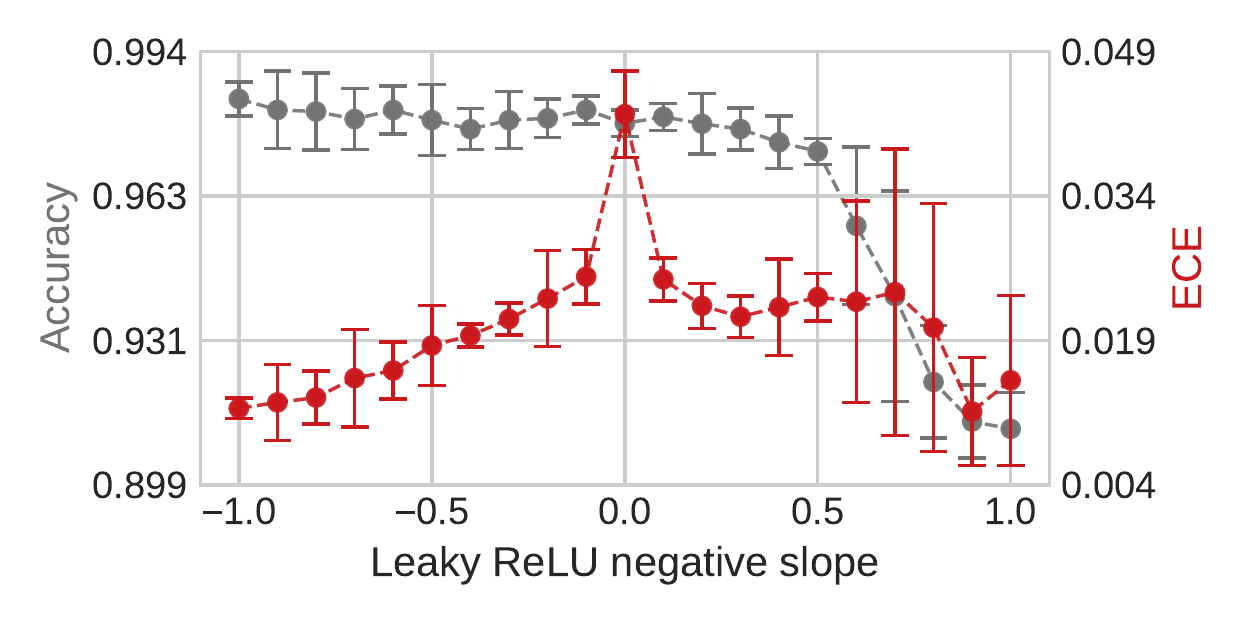}
    \caption{The impact of Leaky ReLU negative slope on ECE and Accuracy for CONV model trained on MNIST.}
    \label{fig:metrics-conv-mnist}
\end{figure}

The experiments were run with the following models:
3FC -- Fully connected neural network with 3 hidden layers of size 1000,
CONV -- convolutional neural network with 2 hidden layers containing 128 and 256 channels respectively, and MNIST\citep{lecun-mnisthandwrittendigit-2010} and Fashion MNIST (FMNIST)\citep{xiao2017fashionmnist} datasets. We compared a family of Leaky ReLU functions with negative slope parameter ranging from $-1$ to $1$. Note that, in particular, this family contains absolute value function, ReLU, and linear activation for slope values of $-1$, $0$, and $1$ respectively. 

We trained Bayesian models using Pyro\citep{bingham2019pyro} MFVI with Adam\citep{kingma2014adam} optimizer, a learning rate of $0.001$, and normally distributed weight priors of mean 0 and variance 1. The quality of uncertainty estimation was measured using Expected Calibration Error (ECE).

\subsection{Shape of the likelihood function} 

We trained deterministic models on all the architectures and datasets and visualized conditional likelihood functions for some random weights in the vicinity of the mode of the likelihood. We observed the phenomena we predicted theoretically for the majority of the weights in all considered architectures with ReLU, and for 10-30\% of them this flat region was near the mode of the distribution (as shown in Fig.~\ref{fig:ll}), thus implying it would affect fitting the posterior using MFVI.

\subsection{ECE dependence on negative slope}

We ran 5 experiments for each configuration of a model, dataset, and Leaky ReLU slope. Results for one combination are shown in Fig.~\ref{fig:metrics-conv-mnist}. In all configurations but one, the highest ECE corresponds to the slope of 0 (which is ReLU). In most cases, the maximum in ECE has a form of a spike, as in Fig.~\ref{fig:metrics-conv-mnist}. 

Moreover, there is no definite domination in terms of accuracy while using ReLU which means that using Leaky ReLU leads to better calibrated models with similar accuracy.

\subsection{Comparison of different architectures and datasets}

\begin{table}[t!]
\centering
\caption{Accuracy and ECE for different architectures, activations and datasets. Leaky stands for Leaky ReLU.}

\begin{tabular}{llrrrr}
\toprule
 & {} & \multicolumn{2}{c}{Accuracy} & \multicolumn{2}{c}{ECE} \\
 &  & Leaky & ReLU & Leaky & ReLU \\
Dataset & Model &   &  &   &  \\
\midrule
FMNIST & CONV &  \textbf{0.83} & \textbf{0.83} &  \textbf{0.014} & 0.079 \\
 & 3FC &  \textbf{0.83} & 0.82 &  \textbf{0.033} & 0.106 \\
MNIST & CONV &  \textbf{0.98} & \textbf{0.98} &  \textbf{0.018} & 0.042 \\
 & 3FC &  0.97 & \textbf{0.98} &  \textbf{0.007} & 0.019 \\
\bottomrule
\end{tabular}

\label{tab:comparison}

\end{table}

In Table~\ref{tab:comparison} we present average results for ECE and Accuracy for all models and datasets for Leaky ReLU negative slope equal $-0.5$ (almost all best results for all combinations were between $-1$ and $-0.5$). We can see that ReLU and Leaky ReLU models are comparable in terms of accuracy and Leaky ReLU models yield much better ECE.

\subsection{Decalibration for ReLU}
\begin{figure}[t!]
    \centering
    \includegraphics[width=0.9\textwidth]{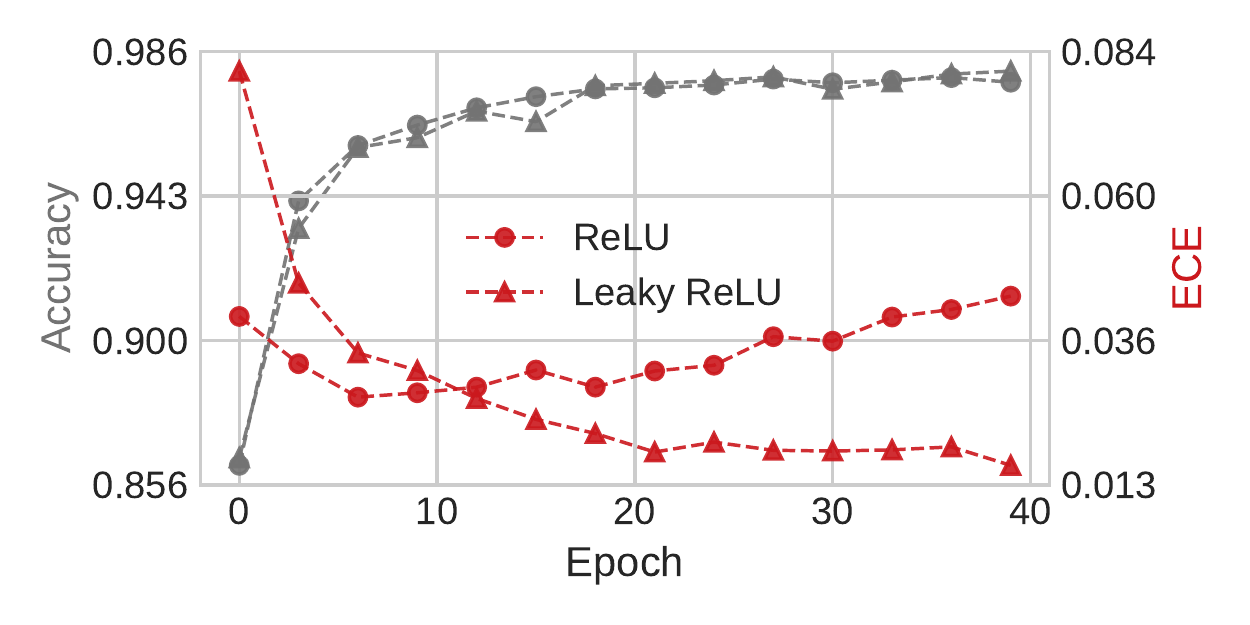}
    \caption{Decalibration during ReLU training for CONV network on MNIST dataset. We plot accuracy and ECE on the validation set and Leaky ReLU model for comparison.}
    \label{fig:decalibration-relu}
\end{figure}

We observed a phenomenon that might also be related to our observation about BNNs using ReLU activations. Figure ~\ref{fig:decalibration-relu} shows measured metrics during the training of a model. While accuracy on the validation set increases during training with both activations, the ECE for ReLU on the validation set initially drops as expected, but then it starts to increase before the model has fully converged. This behavior has been noticed in the vast majority of experiments using ReLU (and only ReLU) activation.

\section{Conclusions}

We have shown that Leaky ReLU leads to superior ECE scores with comparable accuracy in comparison with ReLU activations. We plan to conduct more experiments on larger architectures, bigger and non-image datasets, and verify if our observation holds in case of regression problems.

\bibliography{main}

\end{document}